\documentclass[a4paper,11pt]{article}
\usepackage{fullpage} 

\usepackage{diagbox} 
\usepackage{slashbox} 
\usepackage{tabularx} 
\usepackage{bbm} 
\usepackage{pifont}  

\usepackage{amsmath}
\usepackage{amssymb}
\usepackage{amsthm}
\usepackage{mathtools}
\usepackage{mathdots}
\usepackage{appendix}
\usepackage{graphicx}
\usepackage{enumerate}
\usepackage{microtype}
\usepackage{mleftright}
\usepackage{mdframed}
\usepackage{authblk}
\usepackage{tcolorbox}
\usepackage[english]{babel}
\usepackage[utf8]{inputenc}
\usepackage{algorithm}
\usepackage[noend]{algpseudocode}
\usepackage{xspace}
\usepackage{tikz}
\usepackage{pgfplots}
\pgfplotsset{width=8cm,compat=1.9}
\usepackage{mathdots}
\usepackage{url}
\usepackage{forest}
\forestset{default preamble={for tree={s sep+=1cm}}}
\usepackage{caption,subcaption}
\usepackage[noadjust]{cite}
\usepackage{hyperref}

\tcbset{
    rounded corners,
    colback = white,
    before skip = 0.2cm,
    after skip = 0.5cm,
    boxrule = 1.5pt,
    arc = 5pt
}

\makeatletter
\renewcommand{\ALG@name}{Protocol}
\makeatother
\allowdisplaybreaks

\graphicspath{ {./figures/} }

\usepackage{algorithm}
\usepackage[noend]{algpseudocode}

\newtheorem{theorem}{Theorem}[section]
\newtheorem*{theorem*}{Theorem}

\newtheorem{lemma}[theorem]{Lemma}

\newtheorem*{corollary*}{Corollary}

\theoremstyle{definition}

\newcommand{\cN}{\mathcal{N}}

\newcommand{\R}{\mathbb{R}}
\newcommand{\C}{\mathbb{C}}

\newcommand{\N}{\mathbb{N}}

\newcommand{\Real}{\operatorname{Re}}

\newcommand{\diag}{\mathsf{diag}}

\newcommand{\sign}{\mathsf{sign}}
\newcommand{\Var}{\operatorname{Var}}

\newcommand{\E}{\mathbb{E}}      

\newif\ifanonymous
\anonymousfalse

\begin{document}
\title{Most Convolutional Networks Suffer from Small Adversarial Perturbations}

\author[1]{Amit Daniely}
\author[1]{Idan Mehalel}
\affil[1]{The Hebrew University}

\maketitle

\begin{abstract}
The existence of adversarial examples is relatively understood for random fully connected  neural networks, but much less so for convolutional neural networks (CNNs). The recent work \cite{daniely2025existence} establishes that adversarial examples can be found in CNNs, in some non-optimal distance from the input. We extend over this work and prove that adversarial examples in random CNNs with input dimension $d$ can be found already in $\ell_2$-distance of order $\lVert x \rVert /\sqrt{d}$ from the input $x$, which is essentially the nearest possible. We also show that such adversarial small perturbations can be found using a single step of gradient descent.
To derive our results we use  Fourier decomposition to efficiently bound the singular values of a random linear convolutional operator, which is the main ingredient of a CNN layer. This bound might be of independent interest.
\end{abstract}


\section{Introduction}
An adversarial example $\tilde{x}$ is a naturally-looking input to a neural network, which is in fact a result of a well-crafted, small perturbation performed on a real natural example $x$, with the goal of making the model give the wrong output. The example $x$ can be, for example, a picture of a panda (represented e.g.\ as a real matrix), and the network can be a classifier sending a picture to the name of the animal appearing in it. While the network may well classify $x$ to ``Panda", it gives a different, incorrect output for $\tilde{x}$. This may sound trivial if $\tilde{x}$ does not resemble $x$. However, interestingly, this phenomenon occurs even when $x$ and $\tilde{x}$ are so similar that to the human eye they seem as two identical pictures (or more generally, inputs). Mathematically speaking, there exist adversarial examples $\tilde{x}$ with very small $\lVert x - \tilde{x} \rVert$.

This phenomenon was first empirically observed by \cite{szegedy2013intriguing}, and many other works in the past decade studied it. This includes attack \cite{athalye2018obfuscated, carlini2017adversarial, carlini2018audio, goodfellow2014explaining, grosse2017statistical} and defense \cite{papernot2016distillation, papernot2017practical, madry2017towards, wong2018provable, feinman2017detecting} methods, as well as some attempts to explain the phenomenon \cite{fawzi2018adversarial, shafahi2018adversarial, shamir2019simple, schmidt2018adversarially, bubeck2019adversarial}.

\subsection{Background and related work}

In the past decade, researchers have aimed to provide a theoretical explanation of this phenomenon. Most related to our work, is the line of work studying the existence of adversarial examples in random networks. The work of \cite{daniely2020most} showed that adversarial examples can be found in random fully connected ReLU networks of input dimension $d$ and constant depth, as long as the width decreases from layer to layer, already in $\ell_2$-distance $\tilde{O}(\lVert x \rVert/\sqrt{d})$ from the input $x$, which is essentially the nearest possible. They further showed that gradient flow finds adversarial examples in such nets. The work \cite{bartlett2021adversarial} improved these results by replacing the decreasing width assumption of \cite{daniely2020most} with a very mild assumption. The work \cite{montanari2023adversarial} further improved this result, removing any assumptions on the width.

The works discussed in the previous paragraph are limited to \emph{fully connected networks}. While adversarial examples are often discussed in the context of picture inputs, usually handled by \emph{convolutional neural networks} (CNNs), this network architecture received much less attention in the theoretical research on adversarial examples. Recently, the work \cite{daniely2025existence} proved that adversarial examples exist also in CNNs, and as in \cite{montanari2023adversarial}, no restriction on the width is assumed. However, this work does not show how to find adversarial examples in CNNs, and the $\ell_2$-distance of the required perturbation is not optimal.

\subsection{Our contribution}
We extend the body of knowledge initiated by \cite{daniely2025existence} on the existence of adversarial examples in CNNs.
Our main contribution is a proof that adversarial examples exist, under some mild assumptions, in random CNNs, with perturbation $\ell_2$-distance matching\footnote{Except that the work of \cite{daniely2020most} considers ReLU activation, while this work considers a wide family of activations that do not include ReLU, but include all the smooth variants of ReLU.} the results of \cite{daniely2020most} for fully connected networks.
This improves over the result of \cite{daniely2025existence} on this criteria.
As mentioned earlier, this $\ell_2$-distance is essentially the smallest possible.
Indeed, fix an input $x\in \R^d$ and let $w\in \R^d$ be a random spherically symmetric vector (say, a Gaussian).
Then w.h.p, finding $\tilde{x}$ such that $\sign(\tilde{x} ^\top w) \neq \sign(x^ \top w)$ requires that $\lVert \tilde{x} - x \rVert = \Omega(\lVert x \rVert / \sqrt{d})$.
We also show that a single step of gradient descent finds an adversarial example with such small perturbation size, while the results of \cite{daniely2025existence} are not as constructive.

On the negative side, our proof requires constant depth and that the width does not increase too much from layer to layer, which is not assumed in \cite{daniely2025existence}.

\subsection{Paper organization}
In Section~\ref{sec:prem} we formally define the architecture we consider and its random initialization. In Section~\ref{sec:results}, we state and explain the main results. In Section~\ref{sec:sketch} we sketch the proof idea.

The remainder of the paper is dedicated to formally proving our results. In Section~\ref{sec:sin-val} we prove bounds on the singular values of a random linear convolutional operator. This is the main technical section that allows our main result and might be of independent interest. In Sections~\ref{sec:frobenius-bounds}-\ref{sec:gradient-robust} we establish some properties of the gradient of a random CNN. In Section~\ref{sec:main-result}, we combine the tools provided in previous sections to prove the main result.

\section{Preliminaries} \label{sec:prem}

\subsection{Architecture}
In this paper, we consider a convolutional network for binary classification, having the following architecture.
Let $t\in \mathbb{N}$, let $d_0, \ldots, d_t \in \N$, and let $G$ be a finite abelian group of large enough size, and in any way not less than $\max\{c_0, t\}$, where $c_0$ is a constant. We relate to $t$ as a constant as well. We denote the minimal and maximal values of $d_0, \ldots, d_t$ by  $d_{\min}$ and $d_{\max}$, respectively. We assume that for every $\ell \in [t]$, we have $d_{\ell-1} \geq c_w\cdot  d_{\ell}$, and also that $d_{\max} \leq |G|$ and $d_{\min} \geq c_G \log |G|$, where $c_w,c_G$ are universal constants.

The network we consider have $t$ convolutional layers, where the input domain to layer $\ell$ is $L^2(G, \mathbb{R}^{d_{\ell-1}})$, and its output domain is $L^2(G, \mathbb{R}^{d_{\ell}})$. In particular, the input domain for the network is $L^2(G, \mathbb{R}^{d})$ where $d:=d_0$, and the output domain of the last convolutional layer is $L^2(G, \mathbb{R}^{d_t})$. We sometimes use the natural flattening of the functions in $L^2(G, \mathbb{R}^{d_{\ell}})$ to a vector, and then instead of considering a function $f \in L^2(G, \mathbb{R}^{d_{\ell}})$, we consider a vector $v \in \mathbb{R}^{N_\ell}$ where $N_\ell = |G| d_{\ell}$. Formally, if $G:= \{g^{(1)}, \ldots, g^{(|G|)}\}$ then $f(g^{(i)}) = v_{(i-1)(d_{\ell}-1) +1}, \ldots, v_{(i-1)(d_{\ell}-1) +d_\ell}$.
To implement binary classification, our network has a final fully connected layer consists of a single vector $u \in \R^{N_t}$. We denote the function computed by the convolutional part of the net by $H: \R^{N_0} \to \R^{N_t}$, and the function computed by the entire network by $H_b: \R^{N_0} \to \R$. The output of the network for the input $f \in L^2(G, \mathbb{R}^{d})$ is thus given by $H_b(f) = \langle u^\top, H(f) \rangle$, and the classification decision of the network for $f$ is given by $\sign(H_b(f))$.

When putting aside the activation, each convolutional layer $\ell$ is a linear convolutional operator (matrix)  $\Lambda_\ell$, parametrized by $n_\ell$ distinct elements $g_1,\ldots, g_{n_\ell}$ of $G$, $n_\ell$ weight matrices $W_1 \ldots, W_{n_\ell} \in \mathbb{R}^{d_\ell \times d_{\ell-1}}$, and defined by:
\[
    (\Lambda_{\ell} f) (g) = \frac{1}{\sqrt{n_\ell}} \sum_{i=1}^{n_\ell} W_i f(g_i g).
\]
Note that $\Lambda_\ell \in \mathbb{R}^{N_{\ell} \times N_{\ell-1}}$, but it is not a general $N_{\ell} \times N_{\ell-1}$ matrix and restricted to be a convolutional operator parametrized by $g_1, \ldots, g_{n_\ell}$ and $W_1, \ldots, W_{n_\ell}$, as defined above. 

The activation function is $\sigma: \mathbb{R} \to \mathbb{R}$. We assume that it is $C^2$ and that $\sigma(0)=0$. We further assume that $\lVert \sigma' \rVert_\infty, \lVert \sigma'' \rVert_\infty = O (1)$ and that for all $r\in \R$ we have $(\sigma'(r))^2 + (\sigma'(-r))^2 \geq 2c^2$, where $c$ is a universal constant. Many common activations including smooth variations of ReLU satisfy those assumptions.

We use the following notation for intermediate outputs of the network. Fix the input $f$ and denote it by $h^{(0)} := f$. For each layer $\ell \in [t]$, let $z^{(\ell)} := \Lambda_{\ell} h^{(\ell-1)}$ be the pre-activation of layer $\ell$, and let $h^{(\ell)} := \sigma(z^{(\ell)})$ be the post-activation of layer $\ell$.

\subsection{Initialization of the network}
Fix any $\ell \in [t]$. Throughout the paper, we assume that $(W_1, \ldots, W_{n_\ell})$ are chosen as follows: Each entry of each matrix is drawn iid from $\cN(0, 1/{d_{\ell-1}})$. In the final fully connected layer $u$, each entry is drawn iid from $\cN(0, 1/N_t)$. Note that this is the standard (Xavier) initialization distribution

\subsection{Characters and Representations}\label{sec:reps}
The material in this section can be found in chapters 2 and 3 of \cite{simon1996representations}.
Throughout this section, let $G$ be a finite abelian group.
A \emph{character} of $G$ is a homomorphism $\chi: G \to \mathbb{C}^\times$. The set of all characters forms the \emph{dual group} $\widehat{G}$. It is a standard result that $|\widehat{G}| = |G|$, that the characters form a group under pointwise multiplication, and that they take values on the unit circle in $\mathbb{C}$. Furthermore, the characters form an orthonormal basis for $L^2(G, \mathbb{C})$, the Hilbert space of complex-valued functions on $G$ equipped with the inner product $\langle u, v \rangle = \E_{x \sim G} [u(x)\overline{v(x)}]$, where the expectation is taken with respect to the uniform measure on $G$.

A \emph{(finite-dimensional real) representation} of $G$ is a real vector space $V$ equipped with a linear action of $G$. A linear map between two representations is called a \emph{homomorphism} (or an equivariant map) if it commutes with the action of $G$. If such a map is also a linear isomorphism, we call it an \emph{isomorphism of representations}.

We say that a representation $V$ is \emph{irreducible} if it admits no proper, non-zero subspace $V' \subset V$ that is invariant under the action of $G$. We construct the irreducible real representations of $G$ using characters as follows. For each character $\chi \in \widehat{G}$, we define the space:
\[
V_{\chi} = \{g \mapsto \mathrm{Re}(c \cdot \chi(g)) \mid c \in \mathbb{C}\} \subset L^2(G, \mathbb{R}).
\]
It is well known that $V_\chi$ constitutes an irreducible representation under the action of translation by $G$. Furthermore, $V_{\chi}$ is isomorphic to $V_{\chi'}$ if and only if $\chi' \in \{\chi, \bar{\chi}\}$. Every irreducible real representation of $G$ is isomorphic to $V_{\chi}$ for some $\chi$.

A representation $V$ is called \emph{isotypic} if all of its irreducible subrepresentations are isomorphic to the same irreducible representation, called the {\em type} of $V$ (and is defined up to isomorphism). An \emph{isotypic component} of a representation $V$ is a maximal (w.r.t.\ inclusion) isotypic subrepresentation. 
It is a fundamental result in representation theory that any representation $V$ is the direct sum of its isotypic components, and that any pair of different non-zero isotypic components correspond to different types. Moreover, any homomorphism between representations maps the isotypic component of a specific type in the domain to the isotypic component of the same type in the codomain.

In this work, we utilize the fact that the isotypic components of the space of vector-valued functions $L^2(G, \mathbb{R}^d)$—viewed as a representation with respect to translation—are 
\[
V_\chi \otimes \R^d = \mathrm{span}\left(\left\{g\mapsto h(g)x : h\in V_{\chi},x\in\R^d \right\}\right)
\]
where the type of $V_\chi \otimes \R^d$ is $V_\chi$.
(note that $V_\chi \otimes \R^d = V_{\chi'} \otimes \R^d$ iff $\chi'\in \{\chi,\bar\chi\}$).

\section{Results} \label{sec:results}
Our main result is the following

\begin{theorem} \label{thm:main-intro}
    Fix $f \in L^2(G,\mathbb{R}^d)$ with $\lVert f \rVert_\infty \leq 1$.
    Then, with probability at least $0.95$, a single step of gradient descent of Euclidean length  $O(1)$ starting from $f$ will reach $f' \in L^2(G,\mathbb{R}^d)$ such that $\sign(H_b(f')) \neq \sign(H_b(f))$.
\end{theorem}
The $O(1)$ term hides constants depending on the architecture: $c, t, \lVert \sigma' \rVert_\infty, \lVert \sigma'' \rVert_\infty$.
In the case where most entries of $f$ are bounded away from $0$, Theorem~\ref{thm:main-intro} implies that an adversarial example $f'$ with $\lVert f - f' \rVert = O(\lVert f \rVert/\sqrt{N_0})$ exists. As explained in the introduction, this is essentially optimal.

As mentioned in Section~\ref{sec:prem}, our proof requires that the activation is $C^2$, which excludes ReLU but includes all smooth variations of it. We conjecture that similar ideas to those used in this work can be used to prove Theorem~\ref{thm:main-intro} for ReLU activation as well. However, in this work we focus on general activations with relatively weak assumptions, and then the $C^2$ requirement is necessary.

As part of our proof of Theorem~\ref{thm:main-intro}, we bound the singular values of a random linear convolutional operator as follows.

\begin{theorem} \label{thm:sin-val-intro}
    Let $d, q \in \N$, and let $\Lambda: L^2(G, \R^d) \to L^2(G, \R^q)$ be a random linear convolutional operator as defined in Section~\ref{sec:prem}. Then, there are universal constants $0<a<b$ and $c_w>0$ such that the following holds. If $d\ge c_wq \ge c_w^2\log(|G|)$ then with probability at least $0.99$ the singular values of $\Lambda$ are all in $[a,b]$.
\end{theorem}

In fact, our proof of Theorem~\ref{thm:main-intro} only requires the upper bound of Theorem~\ref{thm:sin-val-intro}, showing that all singular values are w.h.p at most $b$. However, the lower bound is interesting for two other reasons. First, it nicely completes the theorem. Second, if one assumes that the activation $\sigma$ satisfies $\sigma'(r) \geq c^2$ for all $r \in \R$, then this assumption,  together with the lower bound of Theorem~\ref{thm:sin-val-intro} gives a simpler proof\footnote{In a nutshell: with this assumption, one can show $\lVert \nabla H_b (f') \rVert = \Omega(1)$ for all $f'$ using the lower bound of Theorem~\ref{thm:sin-val-intro}. Then, since typically $\lvert H_b(f) \rvert = O(1)$, the result is implied.} of Theorem~\ref{thm:main-intro} than the one given in this paper. However, this assumption is strictly stronger than ours. For example, smooth ReLU variations $\sigma$ with $\sigma(x) = 0$ for all $x \leq 0$  satisfy our assumption, but not this stronger one.

\section{Proof sketch} \label{sec:sketch}
In this Section, we sketch the proof of a slightly simplified version of Theorem~\ref{thm:main-intro}: we explain how an adversarial example can be found by a gradient flow of length $O(1)$, instead of a single step of gradient descent. The part of the proof showing that a single step of Euclidean length $O(1)$ of gradient descent also finds an adversarial example is slightly more technical, and we leave it for the formal proof, given in Section~\ref{sec:main-result}.

Fix the input $f \in L^2(G,\mathbb{R}^d)$, and let $C,a >0$ be constants (which may depend on the fixed architecture of the net). The proof relies on the following three main lemmas.
\begin{lemma} \label{lem:intro-small-output}
    Suppose that $\lVert f \rVert_\infty \leq 1$. Then with high probability, we have
    \[
    \lvert H_b(f) \rvert \leq a.
    \]
\end{lemma}

\begin{lemma} \label{lem:intro-gradient} 
    With high probability, we have
    \[
    \lVert \nabla H_b(f) \rVert \geq C.
    \] 
\end{lemma}

\begin{lemma} \label{lem:intro-gradient-robustness}
    Let $B(f, 2a/C)$ be the ball of radius $2a/C$ centered at $f$. With high probability, for all $f' \in B(f, 2a/C)$ we have
    \[
    \lVert \nabla H_b(f') - \nabla H_b(f) \rVert \leq C/3.
    \]
\end{lemma}

Indeed, having those three lemmas, it is straightforward to prove that gradient flow of length $O(1)$ finds an adversarial example.
By Lemma~\ref{lem:intro-gradient}, Lemma~\ref{lem:intro-gradient-robustness} and the triangle inequality, we have:
\[
    \lVert \nabla H_b(f') \rVert
    \geq
    \lVert \nabla H_b(f) \rVert - \lVert \nabla H_b(f') - \nabla H_b(f)  \rVert
    \geq
    C - C/3
    =
    2C/3
\]
for all $f'\in B(f, 2a/C)$. Therefore, gradient flow of length $2a/C$ starting from $f$ decreases the output of $H_b$ by at least $\frac{2a}{C} \cdot \frac{2C}{3} = 4a/3   > a$, and thus flips the sign of $H_b$ by Lemma~\ref{lem:intro-small-output}. Lemma~\ref{lem:intro-small-output} follows from a standard concentration argument, so we focus on sketching the proof of the other two lemmas.

\subsection{Proof sketch of Lemma~\ref{lem:intro-gradient}}
The proof of Lemma~\ref{lem:intro-gradient} is pretty technical, so we give here a very high-level sketch of it. The main challenge is to to prove that $\lVert J \rVert_F$ is large, where $J:= J_H(f)$ is the Jacobian of the convolutional part of the network on the input $f$, and $\lVert \cdot \rVert_F$ denotes the Frobenius norm. Having that, since $\nabla H_b (f) = J^\top u$ where $u$ is the final fully connected layer, deducing that $\lVert \nabla H_b(f) \rVert \geq C$ is relatively simple. Observe that $J$ can be formulated as a multiplication of $2t$ many matrices of the form
\[
J = D_t \Lambda_t \ldots D_1 \Lambda_1
\]
where $D_\ell$ is the part representing the non-linear part of layer $\ell$, controlled by the activation. Since we assume that the number of layers $t$ is a constant, lower bounding $\lVert D_\ell \Lambda_\ell v \rVert$ provides a lower bound on $\lVert J v \rVert$, for some arbitrary vector $v$. Then, we use the lower bound on $\lVert J v \rVert$ and the identity
\[
\lVert J \rVert_F^2 = \E _r \mleft[ \lVert J r \rVert^2 \mright]
\]
(see e.g.\ \cite{hutchinson1989stochastic}) where $r \in \{\pm 1\}^{N_0}$ is a random Rademacher vector to lower bound $\lVert J \rVert_F^2$.

\subsection{Proof sketch of Lemma~\ref{lem:intro-gradient-robustness}}

The proof of Lemma~\ref{lem:intro-gradient-robustness} is also quite technical, so we briefly sketch the high-level idea. The first step is to control the difference $\lVert \nabla H_b(f') - \nabla H_b(f) \rVert$ via three separate terms, namely to show that:
\[
\lVert \nabla H_b(f') - \nabla H_b(f) \rVert \leq O \mleft( c_\infty \cdot C_s \cdot C_B \mright)
\]
where $c_\infty$ is a term controlled by the infinity norm of $\nabla H_b(f)$, $C_s$ is controlled by the singular values of the convolutional operators, and $C_B$ is controlled by the change in activation's derivative between $f$ and $f' \in B(f, 2a/C)$. The main idea is that we can show that $C_s,C_B$ are constants, while $c_\infty \to 0$ as $|G| \to \infty$. Having that in hand, the whole expression $\lVert \nabla H_b(f') - \nabla H_b(f) \rVert$ decreases as $|G|$ increases, and in particular is at most $C/3$ when $|G|$ is large enough, as desired. The fact that $c_\infty \to 0$ as $|G| \to \infty$ is implied from properties of gaussians and the initialization of the network. The main part in the proof that $C_s,C_B$ are constants is to upper bound the singular values of a linear convolutional operator, which is precisely what Theorem~\ref{thm:sin-val-intro} gives. As this theorem might also be of independent interest, we also briefly sketch its proof below.

\subsubsection{Proof sketch of Theorem~\ref{thm:sin-val-intro}}
Let $\Lambda: L^2(G, \R^d) \to L^2(G, \R^q)$ be the random convolutional linear operator defined in Section~\ref{sec:prem}. Theorem~\ref{thm:sin-val-intro} claims that this operator, represented as a large $|G|q \times |G| d$ matrix initialized with gaussians with relatively large (and independent of $|G|$) variance, has constant singular values w.h.p. The key idea allowing the proof of Theorem~\ref{thm:sin-val-intro} is to decompose $L^2(G, \R^d)$ into small subspaces of dimension $O(d)$ using a Fourier decomposition, such that $\Lambda$ acts separately on each subspace. Since each subspace is of dimension $O(d)$, we are able to use standard random matrix theory arguments in a sufficiently efficient way, separately on each subspace. Then, we use a union bound over all subspaces to derive the theorem. We now explain how to derive Theorem~\ref{thm:sin-val-intro} using this idea in more detail.

For each character $\chi:G \to \C$ of $G$ we define
\[
V_{\chi} = \{g \mapsto \Real(c \cdot \chi(g)): c \in \mathbb{C}\},
\]
and observe three crucial facts, that hold since $G$ is abelian:
\begin{enumerate}
    \item The dimension of $V_\chi \otimes \R^d$ is at most $2d$.
    \item $L^2(G,\R^d) = \bigoplus_{\chi} (V_\chi \otimes \R^d)$.
    \item $\Lambda(\mathbb{R}^d \otimes V_\chi) \subset \mathbb{R}^q \otimes V_\chi$ where  $\Lambda(\mathbb{R}^d \otimes V_\chi):= \{ \Lambda(f) : f \in \mathbb{R}^d \otimes V_\chi\}$.
\end{enumerate}
Therefore, we may define $\Lambda_\chi : V_\chi \otimes \R^d \to V_\chi \otimes \R^q $ as the restriction of  $\Lambda$ to the domain $V_\chi \otimes \R^d $, and the singular values of $\Lambda$ is the union of the  singular values of $\Lambda_\chi$ for all characters $\chi$ of $G$. Therefore, to prove Theorem~\ref{thm:sin-val-intro}, it suffices to show that for any $\chi$, the singular values of $\Lambda_\chi$ are bounded with probability at most $O(1)/|G|$. Since there are $|G|$ many distinct characters, a union bound over all characters will imply Theorem~\ref{thm:sin-val-intro}. So, it remains to  bound  the singular values of $\Lambda_\chi$ for an arbitrary character $\chi$. This can be done by relatively standard random matrix theory techniques, that are essentially good enough for our purposes since the dimension of $V_\chi \otimes \R^d$ is at most $2d$, and thus can be well-approximated by a small net, of size roughly $(O(1))^d$. 

\section{The singular values of a single linear convolutional layer} \label{sec:sin-val}
We start by considering a single linear convolutional layer. In other words, we consider a single $\Lambda_\ell$ in this section (with no activation). So, we drop the $\ell$ subscript from all notation, and denote the input domain of $\Lambda$ by  $L^2(G, \mathbb{R}^{d})$, and the output domain by  $L^2(G, \mathbb{R}^{q})$. Each $W_i$ parameterizing $\Lambda$ is a $q \times d$ matrix. The number $n$ is the number of offsets, where each $W_i$ operates in offset $i$. The number of input channels is $d$, and $q$ is the number of output channels. Let $g_1,\ldots, g_n \in G$ be distinct elements of $G$, which will function as the offsets. In this section, we use $\sigma_s$ (not to confuse with the activation $\sigma$) for the standard deviation of a gaussian. Recall that $\Lambda: L^2(G,\mathbb{R}^d) \to L^2(G,\mathbb{R}^q)$ is defined as follows:
\[
(\Lambda f) (g) = \frac{1}{\sqrt{n}} \sum_{i=1}^n W_i f(g_i g).
\]

Fix a character $\chi : G \to \mathbb{C}$, and let 
\[
V_{\chi} = \{g \mapsto \Real(c \cdot \chi(g)): c \in \mathbb{C}\}.
\]
We identify $\R^d\otimes V_\chi$ as
\[
\R^d\otimes V_\chi = \left\{ g\mapsto h(g)x : h\in V_\chi,\;x\in\R^d \right\}.
\]
Let $\Lambda_\chi$ be the restriction of $\Lambda$ to inputs from $\R^d\otimes V_\chi$.
We first prove the following, which will be used to bound the singular values of $\Lambda$.

\begin{lemma} \label{lem:restricted-sin-val}
    There exist universal positive constants $a,b,c_w, c_G$ such that $ a <b$, for which the following holds. Let $d$ such that $d \geq c_w q$ and $\min\{d,q\} \geq c_G \log |G|$, and let $\chi : G \to \mathbb{C}$ be a character. Then with probability at least $1-0.01/|G|$, the singular values of $\Lambda_\chi$ are all in $[a,b]$.
\end{lemma}

In order to prove the lemma, we will use the following known result.

\begin{theorem}\label{thm:Gauss_norm_concentration}
Suppose that $X\sim\cN(0,\sigma_s^2I_d)$. Then, $\sigma_s(\sqrt{d} - \Theta(1)) \le \E \|x\|\le  \sigma_s\sqrt{d}$. Furthermore, $\Pr(\left|\|x\|-\E\|x\|\right| > t\sigma_s)\le 2e^{-\frac{t^2}{2}}$
\end{theorem}

We will now prove Lemma~\ref{lem:restricted-sin-val}. Fix a character $\chi$.

\subsection{Upper bound on singular values}
We begin with the upper bound.
Let $d' \leq 2d$ be the dimension of $\R^d\otimes V_\chi$, and let $S^{d' -1}$ be the unit sphere in $\R^d\otimes V_\chi$. Fix $\epsilon \in (0,1/3)$, and let $N_\epsilon$ be an $\epsilon$-net for $S^{d'-1}$ of size at most $\mleft( \frac{3}{\epsilon} \mright)^{2d}$. It is known that such a net exists (see e.g.\ \cite{vershynin2018high}).
First, we prove a bound on $(\Lambda f_0) (g)$ for any $f \in N_\epsilon, g \in G$. Note that $(\Lambda f_0) (g) \sim \cN(0, \sigma_s^2(f_0,g) I_q)$, where $\sigma_s^2(f_0,g) = \frac{1}{nd}\sum_{i=1}^n \|f_0(g_ig)\|^2$.

\begin{lemma} \label{lem:upper-bound-on-net}
    The probability that there exists $f_0 \in N_{\epsilon}$ and $g \in G$ for which $\lVert (\Lambda f_0)(g) \rVert > 2\sigma_s(f_0,g)\sqrt{10d/\epsilon}$ is at most $2e^{-2d/\epsilon}$.
\end{lemma}

\begin{proof}
    Fix $f_0 \in N_{\epsilon}$ and $g \in G$, and let $x = (\Lambda f_0)(g)$, $\sigma_s := \sigma_s(f_0,g)$. As mentioned above, we have $x\sim \cN(0,\sigma_s^2 I_q)$. By Theorem~\ref{thm:Gauss_norm_concentration}, we have $\Pr[\lVert x \rVert > \E[\lVert x \rVert] + t \sigma_s] \leq 2e^{-\frac{t^2}{2}}$, which implies $\Pr[\lVert x \rVert > \sigma_s(\sqrt{q} + t)] \leq 2e^{-\frac{t^2}{2}}$. Choose $t = \sqrt{10 \ d/\epsilon}$, since $d \geq c_w q$ by assumption, we conclude:
    \[
    \Pr \mleft[ \lVert x \rVert > 2 \sigma_s \cdot \sqrt{10 d/\epsilon}  \mright]
    \leq
    \Pr \mleft[\lVert x \rVert > \sigma_s (\sqrt{q} + \sqrt{10d/ \epsilon}) \mright]
    \leq
    2e^{-5d/\epsilon},
    \]
    which holds for all $c_w \geq \epsilon/10$.
    By a union bound, the probability that there exist $f_0 \in N_{\epsilon}$ and $g \in G$ for which $\lVert (\Lambda f_0)(g) \rVert > 2\sigma_s(f_0,g)\sqrt{10d/\epsilon}$ is at most
    \[
    |N_\epsilon| |G| 2e^{-5d/\epsilon} \leq 2e^{2d \log (3/\epsilon) + d/c_w^2 - 5d/\epsilon} \leq 2e^{-2d/\epsilon},
    \]
    which holds for all $c_w \geq 1/2$.
\end{proof}

We can now prove the upper bound.

\begin{lemma}\label{lem:sin-upper-bound}
    Assuming that the bad event of Lemma~\ref{lem:upper-bound-on-net} does not hold, we have $\lVert \Lambda_\chi \rVert \leq \frac{\sqrt{40/\epsilon}}{1-\epsilon}$.
\end{lemma}

\begin{proof}
    Fix $f_0 \in N_\epsilon$. We first show that $\lVert \Lambda f_0 \rVert$ is small. Recall that $\lVert (\Lambda f_0)(g) \rVert \leq 2 \sigma_s(f_0, g) \sqrt{10d/\epsilon}$ for all $g \in G$. Therefore, by definition:
    \[
    \lVert \Lambda f_0 \rVert^2 = \sum_{g \in G} \lVert (\Lambda f_0)(g) \rVert^2 \leq \frac{40 d}{\epsilon} \sum_{g \in G} \sigma_s^2(f_0,g).
    \]
    Again, by definition:
    \[
    \sum_{g \in G} \sigma_s^2(f_0,g)
    =
    \sum_{g \in G} \frac{1}{nd}\sum_{i=1}^n \|f_0(g_ig)\|^2
    =
    \frac{1}{nd}\sum_{i=1}^n \sum_{g \in G} \|f_0(g_ig)\|^2
    =
    1/d,
    \]
    since for each $i$, the mapping $g \mapsto g_ig$ is a permutation of $G$ and therefore $\sum_{g \in G} \|f_0(g_ig)\|^2 = \lVert f_0 \rVert^2 =1$.
    Combining the two inequalities gives $\lVert \Lambda f_0 \rVert \leq \sqrt{40/\epsilon}$. Now, we use the $\epsilon$-net to bootstrap the bound (in a slightly weaker form) to hold for $\lVert \Lambda_\chi \rVert$, instead of only for $\sup_{f_0 \in N_\epsilon} \lVert \Lambda f_0 \rVert$. For that, we use the known (see e.g.\ \cite{vershynin2018high}) inequality
    \[
    \lVert \Lambda_\chi \rVert \leq \frac{1}{1-\epsilon} \sup_{f_0 \in N_\epsilon} \lVert \Lambda f_0 \rVert \leq \frac{\sqrt{40/\epsilon}}{1-\epsilon},
    \]
    as required.
\end{proof}

We note here that for the upper bound to hold, we just need $c_w \geq 1/2$, as shown in the proof of Lemma~\ref{lem:upper-bound-on-net}. As explained in Section~\ref{sec:sketch}, we only need the upper bound on the singular values for our proof of Theorem~\ref{thm:main-intro} to go through, and the lower bound is proved for completeness. That is, assuming $d \geq q/2$ suffices. The lower bound proved in the next section, requires a larger $c_w$ value.

\subsection{Lower bound on singular values}
We assume that the upper bound proved in the previous section indeed holds, which happens with high probability.
For the lower bound, it will be convenient to consider $\Lambda^{\top}$ instead of $\Lambda$. As in the previous section, we denote the restriction of $\Lambda^{\top}$ to $\R^q\otimes V_\chi$ by $\Lambda^\top_\chi$.
Let $q' \leq 2q$ be the dimension of $\R^q\otimes V_\chi$, and let $S^{q' -1}$ be the unit sphere in $\R^q\otimes V_\chi$. Let $N_\delta$ be a $\delta$-net for $S^{q'-1}$ of size at most $\mleft( \frac{3}{\delta} \mright)^{2q}$, where $\delta \in (0,1/3]$. It is known that such a net exists.

As we did for the upper bound, we first prove a lower bound on $(\Lambda^\top f_0) (g)$ for any $f \in N_\delta, g \in G$. Note that $(\Lambda^\top f_0) (g) \sim \cN(0, \sigma_s^2(f_0,g) I_d)$, where here, $\sigma_s^2(f_0,g) = \frac{1}{nd}\sum_{i=1}^n \|f_0(g_i^{-1}g)\|^2$.

\begin{lemma} \label{lem:lower-bound-on-net}
    Suppose that $\delta = 0.001$ and $c_w=20000$. The probability that there exists $f_0 \in N_{\delta}$ and $g \in G$ for which $\lVert (\Lambda^\top f_0)(g) \rVert < \sigma_s(f_0,g)\sqrt{d/4}$ is at most $2e^{-d/400}$.
\end{lemma}

\begin{proof}
    Fix $f_0 \in N_{\delta}$ and $g \in G$, and let $x = (\Lambda^{\top} f_0)(g)$, $\sigma_s := \sigma_s(f_0,g)$. As mentioned above, we have $x\sim \cN(0,\sigma_s^2 I_d)$. By Theorem~\ref{thm:Gauss_norm_concentration}, we have $\Pr[\lVert x \rVert < \E[\lVert x \rVert] - t \sigma_s] \leq 2e^{-\frac{t^2}{2}}$, which implies $\Pr[\lVert x \rVert < \sigma_s(\sqrt{d/2} - t)] \leq 2e^{-\frac{t^2}{2}}$.
    Choose $t = \sqrt{d/(100)}$, and then:
    \[
    \Pr \mleft[ \lVert x \rVert < \sigma_s \sqrt{ d/4}  \mright]
    \leq
    \Pr \mleft[\lVert x \rVert < \sigma_s (\sqrt{d/2} - \sqrt{d/(100)}) \mright]
    \leq
    2e^{-d/200}.
    \]
    By a union bound, the probability that there exist $f_0 \in N_{\delta}$ and $g \in G$ for which $\lVert (\Lambda^\top f_0)(g) \rVert < \sigma_s(f_0,g) \sqrt{ d/4} $ is at most
    \[
    |N_\delta| |G| 2e^{-d/200} \leq 2e^{2(d/c_w) \log (3/\delta) + d/c_w^2 - d/200} \leq 2e^{-d/400},
    \]
    where the last inequality holds for our choice of $c_w, \delta$.
\end{proof}

We now prove the desired lower bound with the choice of $\delta,c_w$ in Lemma~\ref{lem:lower-bound-on-net}. First, we will need the following claim. Let $s_{\min}(N_\delta) := \inf_{f_0 \in N_\delta} \lVert \Lambda^\top f_0 \rVert$.

\begin{lemma} \label{lem:lower-bound-in-terms-of-net}
    We have $s_{\min}(\Lambda^\top_\chi) \geq s_{\min}(N_\delta) - \frac{\sqrt{40/\epsilon}}{1-\epsilon} \delta$.
\end{lemma}

\begin{proof}
    Let $f \in S^{q'-1}$ and let $f_0 \in N_\delta$ so that $\lVert f -f_0 \rVert \leq \delta$. Then by the triangle inequality and the proved upper bound on $\lVert \Lambda \rVert$:
    \[
    \lVert \Lambda^{\top} f_0 \rVert
    \leq
    \lVert \Lambda^{\top} (f_0 -f) \rVert + \lVert \Lambda^{\top}( f ) \rVert
    \leq
    \lVert \Lambda^{\top} \rVert \delta + \lVert \Lambda^{\top}( f ) \rVert
    \leq
    \frac{\sqrt{40/\epsilon}}{1 -\epsilon} \delta + \lVert \Lambda^{\top}( f ) \rVert.
    \]
    Therefore:
    \[
    \lVert \Lambda^{\top}( f ) \rVert
    \geq
    \lVert \Lambda^{\top} f_0 \rVert - \frac{\sqrt{40/\epsilon}}{1 -\epsilon} \delta
    \geq 
    s_{\min}(N_\delta) - \frac{\sqrt{40/\epsilon}}{1 -\epsilon} \delta,
    \]
    as claimed.
\end{proof}

We may now lower bound  $s_{\min}(\Lambda^\top_\chi)$.

\begin{lemma}\label{lem:sin-lower-bound}
    Fix $\epsilon = 0.1$. Assuming the bad event of Lemma~\ref{lem:lower-bound-on-net} does not occur, we have $ s_{\min}(\Lambda^\top_\chi) \geq 0.05$.
\end{lemma}

\begin{proof}
    Due to Lemma~\ref{lem:lower-bound-in-terms-of-net}, it suffices to lower bound $s_{\min}(N_\delta)$ by some constant much larger than $\frac{\sqrt{40/\epsilon}}{1 -\epsilon} \delta$, and then we are done. The proof is very similar to the upper bound from the previous section.
    Fix $f_0 \in N_\delta$. By definition, and the assumed lower bound $\lVert ((\Lambda^\top f_0) (g) \rVert \geq \sigma_s(f_0,g) \sqrt{d/4}$ for all $f_0 \in N_\delta$ and $g\in G$, we have:
    \[
    \lVert \Lambda^{\top} f_0 \rVert^2 = \sum_{g \in G} \lVert (\Lambda^{\top} f_0)(g) \rVert^2 \geq  \frac{d}{16}  \sum_{g \in G} \sigma_s^2(f_0,g).
    \]
    By a similar argument to the one used in the upper bound, we have $\sum_{g \in G} \sigma_s^2(f_0,g) = 1/d$, so $\lVert \Lambda^{\top} f_0 \rVert^2 \geq 1/16$, and so $s_{\min}(N_\delta) \geq 1/4$. Combined with Lemma~\ref{lem:lower-bound-in-terms-of-net}, this implies that:
    \[
    s_{\min}(\Lambda^\top_\chi) \geq 1/4 - \frac{\sqrt{40/\epsilon}}{1 -\epsilon} \delta \geq 0.05,
    \]
    where the second inequality holds for $\delta = 0.001, \epsilon = 0.1$.
\end{proof}

We can now prove Lemma~\ref{lem:restricted-sin-val}. That is, we conclude that for our choice of constants, the singular values of $\Lambda$ restricted to a specific character are, with sufficiently high probability, bounded between the two universal constants $a=0.05, b=30$.

\begin{proof}[Proof of Lemma~\ref{lem:restricted-sin-val}]
 Choose $a=0.05, b=30, c_w=20000, c_G = c_w^2$. The sum of failure probabilities of Lemma~\ref{lem:upper-bound-on-net} and Lemma~\ref{lem:lower-bound-on-net} is $2e^{-2d/\epsilon} + 2e^{-d/400}$. Thus, choosing $\epsilon=0.1$ gives that the bounds of both Lemma~\ref{lem:sin-upper-bound} and Lemma~\ref{lem:sin-lower-bound} hold with probability at least
 \[
 1- 2e^{-20d} - 2e^{-d/400} \geq 1 - 2e^{-d/500} \geq 1 - 2e^{-(20000)^2 \log |G|/500} > 1-0.01/|G|,
 \]
 concluding the proof.
\end{proof}

\subsection{Bounds on the unrestricted operator}
In this section, we want to use Lemma~\ref{lem:restricted-sin-val} to prove the following bounds on the singular values of $\Lambda$ without any restrictions. We will need the following.

\begin{lemma} \label{lem:op-maintains-block}
    Let $\chi: G \to \mathbb{C}$ be a character, and let $\Lambda(\mathbb{R}^d \otimes V_\chi):= \{ \Lambda(f) : f \in \mathbb{R}^d \otimes V_\chi\}$.  Then, $\Lambda(\mathbb{R}^d \otimes V_\chi) \subset \mathbb{R}^q \otimes V_\chi$.
\end{lemma}

\begin{proof}
For $f:G\to Y$ and $g\in G$ we denote by $A_gf$ the function from $G$ to $Y$ given by $(A_gf)(h) = f(g^{-1}h)$.
Note that $A_gA_f = A_{gf}$.  Define also $\Delta_W:L^2(G,\R^d)\to:L^2(G,\R^q)$ by $(\Delta_Wf)(g)= Wf(g)$. We note that
\[
\Lambda = \frac{1}{\sqrt{n}}\sum_{i=1}^n \Delta_{W_i}A_{g_i^{-1}}
\]
Since both $\Delta_{W}$ and $A_{g}$ commute with $G$, $\Lambda$ commutes with $G$. 
The lemma now follows form the fact that $\mathbb{R}^d \otimes V_\chi$ and $\mathbb{R}^q \otimes V_\chi$ are the isotypic components of $L^2(G,\mathbb{R}^d)$ and $L^2(G,\mathbb{R}^d)$ corresponding to $\chi$, and that a commuting operator must map any isotypic component into an isotypic components of the same type (see section \ref{sec:reps})
\end{proof}

We are now ready to prove the main theorem.
\begin{theorem} \label{thm:sin-val}
    There exist universal positive constants $a,b,c_w, c_G$ such that $a <b$, for which the following holds. If $d \geq c_w q$ and $\min\{d,q\} \geq c_G \log |G|$, then with probability at least $0.99$ the singular values of $\Lambda$ are all in $[a,b]$.
\end{theorem}

\begin{proof}
    We have $L^2(G,\R^d) = \bigoplus_{\chi} V_\chi \otimes \R^d$. Therefore, combined with Lemma~\ref{lem:op-maintains-block}, it suffices to prove for all $\chi$, the singular values of the restriction of $\Lambda$ to $V_\chi \otimes \R^d$ has all singular values in $[a,b]$. Since $G$ is abelian, the number of distinct characters of $G$ is precisely $|G|$. Given the assumptions $d\ge c_w q$, $\min\{d,q\} \geq c_G \log |G|$ and Lemma~\ref{lem:restricted-sin-val}, a union bound over all $\chi$ shows that this is the case with probability at least $0.99$, as desired.
\end{proof}

\section{Lower bounding the network's gradient norm} \label{sec:frobenius-bounds}
In this section, we are first interested in tracking the Frobenius norm of the Jacobian of the function $H$ computed by the convolutional part of the network. Using that, we can lower bound the norm of the gradient of the function $H_b$ computed by the entire net. Recall the following  assumption on the activation: for all $r\in \R$, we have $(\sigma'(r))^2 + (\sigma'(-r))^2 \geq 2c^2$, where $c$ is a universal constant.

Fix $f \in L^2(G, \mathbb{R}^d)$ as the input. Let $J:=J_H(f)$ be the Jacobian of $H$ at $f$.
Since the bounds we prove are independent of the $n_\ell$ values, we assume they are all equal to some $n$. So, layer $\ell \in [t]$ is given by $n$ many $d_{\ell} \times d_{\ell-1}$ matrices and elements from $G$ that together define the operator $\Lambda_{\ell}$. In this section, we mostly view $\Lambda_\ell$ as a large $N_{\ell} \times N_{\ell-1}$ matrix where $N_{\ell} = |G| d_{\ell}$.
Let $D_{\ell} := \diag(\sigma'(z_1^{(\ell)}), \ldots, \sigma'(z_{N_\ell}^{(\ell)}))$, and note that
\[
J = D_t \Lambda_t \ldots D_1 \Lambda_1. 
\]
Denote $J_{\ell} = D_\ell \Lambda_\ell$.

The following is the key lemma.

\begin{lemma} \label{lem:gaussian-rows-high-prob}
    Fix $u, v \in \mathbb{R}^{m}$ and let $\alpha_1,\ldots,\alpha_q \sim \cN(0, \tau^2 I_m)$ be i.i.d.
    Then, with probability at least $1-e^{-0.005q}$,
    \[
    \sum_{i=1}^q \sigma'(\alpha_i^\top u)^2 (\alpha_i^\top v)^2
    \ge 0.1 q c^2 \tau^2 \lVert v \rVert^2.
    \]
\end{lemma}

\begin{proof}
    Let $s_1,\ldots,s_q$ be i.i.d.\ Rademacher random variables independent of the $\alpha_i$'s. By the symmetry of the Gaussian distribution, the sequence $(\alpha_i)_{i=1}^q$ is identically distributed to $(s_i \alpha_i)_{i=1}^q$. Thus, it suffices to lower bound the sum with terms $T_i := \sigma'(s_i \alpha_i^\top u)^2 (\alpha_i^\top v)^2$ (using $s_i^2=1$).

    For each $i$, consider the event $E_i$ that both $(\alpha_i^\top v)^2 \ge \tau^2 \|v\|^2$ and $\sigma'(s_i \alpha_i^\top u)^2 \ge c^2$.
    First, since $\alpha_i^\top v \sim \cN(0, \tau^2 \|v\|^2)$, the probability that $(\alpha_i^\top v)^2 \ge \tau^2 \|v\|^2$ is the probability that a standard Gaussian squared exceeds $1$, which is $> 0.3$.
    Second, conditioned on $\alpha_i$, the term $\sigma'(s_i \alpha_i^\top u)^2$ takes one of two values, $\sigma'(\pm \alpha_i^\top u)^2$, at least one of which is $\ge c^2$ by assumption. Thus, $\Pr(\sigma'(s_i \alpha_i^\top u)^2 \ge c^2 \mid \alpha_i) \ge 0.5$.
    By independence, $\Pr(E_i) \ge 0.3 \cdot 0.5 = 0.15$.

    Whenever $E_i$ holds, we have $T_i \ge c^2 \tau^2 \|v\|^2$. By Hoeffding's inequality, the number of indices $i$ for which $E_i$ occurs is at least $0.1q$ with probability at least $1 - e^{-2(0.15-0.1)^2 q} = 1 - e^{-0.005q}$. The claim follows.
\end{proof}

\subsection{Step 1: One direction, one layer}

The first step is showing that adding a layer of the Jacobian does not decrease the norm of any vector by too much. This can later be used to lower bound the Frobenius norm of the Jacobain of the entire covolutional part of the network.

Fix a layer $\ell$. The rows of $\Lambda_\ell$ are denoted by $\{a_r\}_{r=1}^{N_\ell}$.
Recall that $h^{(\ell)} = \sigma(z^{(\ell)})$ where $z^{(\ell)} = \Lambda_\ell h^{(\ell-1)}$. For a fixed row $a_r$, the $0$-entries in $a_r$ are given deterministically by the choice of the offsets $g_1, \ldots, g_n$, so $a_r$ is determined by a vector $\alpha \sim \cN(0, \frac{1}{m_{\ell}} I_{m_{\ell}})$ where $m_{\ell} = n d_{\ell-1}$.

\begin{lemma} \label{lem:frobenius-oned-onel-whp}
    Fix $h^{(\ell-1)}$ and a vector $v \in \R^{N_{\ell-1}}$. Then with probability at least $1- |G| e^{-0.005d_\ell}$:
    \[
    \lVert D_\ell \Lambda_\ell v \rVert^2 \geq 0.1 c^2 \frac{d_\ell}{d_{\ell-1}} \lVert v \rVert^2.
    \]
\end{lemma}

\begin{proof}
    By definition:
    \[
    \lVert D_\ell \Lambda_\ell v \rVert^2
    =
    \sum_{r=1}^{N_{\ell}} \sigma'(z_r^{(\ell)})^2 (a_r^\top v)^2
    =
    \sum_{r=1}^{N_{\ell}} \sigma'(a_r^{\top} h^{(\ell-1)})^2 (a_r^\top v)^2.
    \]
    Now, note that  $a_r^\top v = a_r^\top v_r$, where $v_r$ is defined as follows:
    \[
    (v_r)_i = \begin{cases}
              v_i & (a_r)_{i} \neq  0 \\
              0   & (a_r)_{i} = 0.
              \end{cases}
    \]
    That is, $v_r$ is precisely $v$, except for that each index $i$ of a $0$-entry in $a_r$ is also zeroed in $v_r$.
    Now, each row index corresponds to a pair $(g,j) \in G \times [d_\ell]$. We denote the row corresponding to $(g,j)$ by $a_{g,j}$.
    Note that for a fixed $g$, all $((g,j))_{j=1}^{d_\ell}$ have $0$ in the exact same locations, determined only by the offsets $g_1, \ldots, g_n$. So, for each $g$, we define $v_g$ just as $v$, with the only difference that we put $0$ in the same indices where the vectors $((g,j))_{j=1}^{d_\ell}$ are zeroed. Likewise, define $h_g$ to be as $h^{(\ell-1)}$, only with those entries zeroed. With this notation, by definition we have
    \[
    \lVert D_\ell \Lambda_\ell v \rVert^2
    =
    \sum_{g \in G} \sum_{j=1}^{d_\ell} \sigma'(a^\top_{g,j} h_g)^2 (a^\top_{g,j} v_g)^2.
    \]
    Now fix $g \in G$. Each of the non-zero entries of the rows $(a_{g,j})_{j=1}^{d_\ell}$ are determined by $\alpha_j \sim \cN(0, \frac{1}{m_\ell} I_{m_\ell})$ where $m_\ell = nd_{\ell -1}$. Let $\tilde{v}_g, \tilde{h}_g$ be just as $v_g,h_g$, but with $0$ entries removed. Then:
    \[
    a^\top_{g,j} h_g = \alpha^\top_j \tilde{h}_g, \quad \alpha^\top_{g,j} v_g = \alpha^\top_j \tilde{v}_g.
    \]
    From lemma \ref{lem:gaussian-rows-high-prob} we get that with probability at least $1-e^{-0.005 d_\ell}$:
    \[
    \sum_{j=1}^{d_\ell} \sigma'(a^\top_{g,j} h_g)^2 (a^\top_{g,j} v_g)^2
    =
    \sum_{j=1}^{d_\ell} \sigma'(\alpha^\top_j \tilde{h}_g)^2 (\alpha^\top_j \tilde{v}_g)^2
    \geq
    0.1 d_\ell c^2 \frac{1}{m_\ell} \mleft \lVert \tilde{v}_g \mright \rVert^2.
    \]
    A union bound now gives that with probability at least $1- |G|e^{-0.005 d_\ell}$:
    \[
    \lVert D_\ell \Lambda_\ell v \rVert^2
    \geq
    0.1 d_\ell c^2 \frac{1}{m_\ell} \sum_{g \in G} \lVert  \tilde{v}_g \rVert^2.
    \]
    If $\sum_{g \in G} \lVert \tilde{v}_g \rVert^2 = n \lVert v \rVert^2$, then the above inequality implies the statement of the lemma. So, it remains to prove this combinatorial identity. Fix a coordinate in $v \in \mathbb{R}^{N_{\ell-1}}$, which corresponds to $\hat{g} \in G$ and some channel in $\hat{j} \in [d_{\ell-1}]$. Since $G$ is a group, we have $\hat{g} = g_i g$ for exactly one $g \in G$, for every $i \in [n]$. Therefore, the coordinate $(\hat{g}, \hat{j})$ is not zeroed in precisely $n$ many $\tilde{v}_g$, which concludes the proof.
\end{proof}

\subsection{Step 2: One direction, multiple layers}

We now use Step 1 to obtain a bound on $\lVert J v \rVert$ for some fixed direction $v$.

\begin{lemma}\label{lem:frobenius-oned-multil-whp}
    Fix a vector $v \in \R^{N_{0}}$. Then with probability at least $1- |G| t e^{-0.005 d_{\min}}$ we have
    \[
    \lVert J_H(f) v \rVert^2 \geq (0.1c^2)^{t} \frac{d_t}{d_{0}} \lVert v \rVert^2.
    \]
\end{lemma}

\begin{proof}
    Denote $g^{(0)}=v$, and $g^{(\ell)} = J_\ell \ldots J_1 v$. Therefore $g^{(t)} = J_H(f) v$.
    Since $f,v$ are fixed, fixing $\Lambda_1, \ldots, \Lambda_{\ell-1}$ fixes $g^{(\ell-1)}$. Note that $g^{(\ell)} = J_\ell g^{(\ell-1)} = D_\ell \Lambda_\ell g^{(\ell-1)}$, and therefore if $\Lambda_1, \ldots, \Lambda_{\ell-1}$ are fixed, then Lemma~\ref{lem:frobenius-oned-onel-whp} implies that with probability at least $1- |G| e^{-0.005d_\ell}$ we have
    \begin{equation} \label{eq:layer-bound}
    \lVert g^{(\ell)} \rVert^2 \geq 0.1c^2 \frac{d_\ell}{d_{\ell-1}} \lVert g^{(\ell-1)} \rVert^2.
    \end{equation}
    By a union bound, the inequality \eqref{eq:layer-bound} holds for all $\ell \in [t]$ with probability at least $1- t |G| e^{-0.005d_{\min}}$.
    Telescoping all $t$ inequalities gives that with probability at least $1- t |G| e^{-0.005d_{\min}}$ we have
    \[
    \lVert g^{(t)} \rVert^2
    \geq
    (0.1c^2)^{t} \frac{d_t}{d_{0}} \lVert g^{(0)} \rVert^2
    =
    (0.1c^2)^{t} \frac{d_t}{d_{0}} \lVert v \rVert^2,
    \]
    as required.
\end{proof}

\subsection{Step 3: Multiple directions, multiple layers (Frobenius norm)}

\begin{theorem} \label{thm:frobenius-jacobian-whp}
    With probability at least $1- 2|G| t e^{-0.005 d_{\min}}$ we have
    \[
    \lVert J \rVert_F^2 \geq \frac{1}{2} (0.1 c^2)^t d_t |G|.
    \]
\end{theorem}

\begin{proof}
    We use the identity
    \begin{equation}\label{eq:rademacher-identity}
       \lVert J \rVert_F^2 = \E_r \mleft[ \lVert J r \rVert^2 \mright] 
    \end{equation}
    (see e.g.\ \cite{hutchinson1989stochastic}), where $r \in \{\pm 1\}^{N_0}$ is a random Rademacher vector.
    For any fixed $r$, Lemma~\ref{lem:frobenius-oned-multil-whp} implies
    \[
    \Pr\mleft[ \lVert Jr \rVert^2 \geq (0.1c^2)^t \frac{d_t}{d_0} \lVert r \rVert^2 \mright]
    =
    \Pr\mleft[ \lVert Jr \rVert^2 \geq (0.1c^2)^t d_t |G| \mright]
    \geq
    1- |G| t e^{-0.005 d_{\min}}
    \]
    since $\lVert r \rVert^2 = N_0 = d_0 |G|$.
    Therefore:
    \begin{equation} \label{eq:jacobian-and-rademacher-small-bound}
    \Pr_{J,r} \mleft[ \lVert J r \rVert^2 < (0.1c^2)^t d_t |G|] \mright] < |G| t e^{-0.005 d_{\min}},
    \end{equation}
    when the probability is taken both over $J$ (that is, over $\Lambda_1, \ldots, \Lambda_t$) and over $r$.
    Now, for a fixed draw $W:= (\Lambda_1, \ldots, \Lambda_t)$, let $J(W)$ be the Jacobian given by this draw, and denote
    \[
    S(W) := \Pr_r \mleft[\lVert J(W) r \rVert^2 < (0.1c^2)^t d_t |G| \mright].
    \]
    In words, $S(W)$ is the fraction of random vectors $r \in \{\pm 1\}^{N_0}$ that are ``bad" with the Jacobian $J(W)$, that is, for which $\lVert J(W) r \rVert^2 $ is small.
    So from \eqref{eq:jacobian-and-rademacher-small-bound}, we have
    \[
    \E_W \mleft[S(W) \mright] = \Pr_{J,r} \mleft[ \lVert J r \rVert^2 < (0.1c^2)^t d_t |G|] \mright] < |G| t e^{-0.005 d_{\min}}.
    \]
    Using Markov's bound, we may bound the mass of weights $W$ for which $S(W)$ is at least $1/2$:
    \[
    \Pr_W \mleft [ S(W) > 1/2 \mright ] \leq 2 \E_W \mleft[S(W) \mright] \leq 2|G| t e^{-0.005 d_{\min}}.
    \]
    By definition, for all $W$ with $S(W) \leq 1/2$ we have
    \begin{equation} \label{eq:good-W-bound}
        \Pr_r \mleft[\lVert J(W) r \rVert^2 \geq (0.1c^2)^t d_t |G| \mright] \geq 1/2.
    \end{equation} 
    Therefore, by \eqref{eq:rademacher-identity}, \eqref{eq:good-W-bound} we have
    \begin{align*}
    \lVert J(W)  \rVert^2_F
    &=
    \E_r \mleft[ \lVert J(W) r \rVert^2 \mright] \\
    &\geq
    (0.1c^2)^t d_t |G| \cdot \Pr \mleft[ \lVert J(W) r \rVert^2 \geq (0.1c^2)^t d_t |G| \mright] \\
    &\geq
    \frac{1}{2}(0.1c^2)^t d_t |G|,
    \end{align*}
    which concludes the proof.
\end{proof}

\subsection{Adding the final layer}

We now handle the final fully connected layer and derive a bound on $\lVert \nabla H_b(f) \rVert$. Recall that the final fully connected layer is given as a vector $u \sim \cN(0, \frac{1}{d_t|G|} I_{d_t|G|})$, so the output of the entire net on $f$ is given by
\[
H_b(f) := \langle u, H(f) \rangle = \langle u, h^{(t)}(f) \rangle.
\]
We will use the following technical claim.

\begin{lemma} \label{lem:technical-gradient}
    Let $Z_1, \ldots, Z_r \sim \cN(0,1)$ iid, let $w_1, \ldots, w_r \geq 0$ so that  $\sum_i w_i = 1$, and denote $Y = \sum_i w_i Z_i^2$. Then for all $s \in (0,1)$:
    \[
    \Pr[Y \leq s] \leq \sqrt{s} e^{(1-s)/2}.
    \]
\end{lemma}

\begin{proof}
    Fix $s \in (0,1)$. For all $t > 0$ we have:
    \begin{align*}
        \Pr[Y \leq s] &= \Pr[e^{-tY} \geq e^{-ts}] \\
        &\leq e^{t s} \E[e^{-tY}] \\
        &= e^{t s} \prod_{i=1}^r \E[e^{-tw_i Z_i^2}] \\
        &= \frac{e^{t s}}{\sqrt{\prod_{i=1}^r (1 + 2tw_i)}} \\
        & \leq \frac{e^{t s}}{\sqrt{1 + \sum_{i=1}^r 2tw_i}} \\
        &= \frac{e^{t s}}{\sqrt{1 + 2t}}.
    \end{align*}
    
    The second line is by Markov. The third line is due to independence of $Z_1, \ldots, Z_r$. The fourth line is a standard calculation. The fifth line is due to the known inequality $\prod_{i=1}^r (1 + a_i) \geq 1 + \sum_{i=1}^r a_i$ that holds if $a_1,\ldots, a_r \geq 0$. The last line is since $\sum_i w_i = 1$.

    Now, substitute $t:= \frac{1-s}{2s}$ (which minimizes the expression in the last row) to the above inequality, and the statement is implied.
\end{proof}

We may now prove the main lemma establishing a lower bound on the gradient's norm.

\begin{lemma} \label{lem:gradient} 
    With probability at least $0.99 -  2|G| t e^{-0.005 d_{\min}}$ we have
    \[
    \lVert \nabla H_b(f) \rVert^2 \geq \frac{0.0001(0.1 c^2)^t }{2e}.
    \] 
\end{lemma}

\begin{proof}
    First, the chain rule implies that $\nabla H_b(f) = J^\top u$. Denote $u = \frac{1}{\sqrt{N_t}} z$ where $z \sim \cN(0, I_{N_t})$. So we have $\nabla H_b(f) = \frac{1}{\sqrt{N_t}} J^\top z$, and likewise $\lVert \nabla H_b(f) \rVert^2 = \frac{1}{N_t} \lVert J^\top z \rVert^2$.
    Let $U \Sigma V^\top = J$ be the SVD of $J$. Therefore, we have $J^\top z = V\Sigma^\top U^\top z$. As $V$ is orthogonal, we have $\lVert J^\top z \rVert = \lVert \Sigma^\top U^\top z \rVert$. From rotational invariance of a gaussian combined with orthogonality of $U$, we have $U^\top z \sim z$. Thus,  it holds that $\lVert \Sigma^\top U^\top z \rVert \sim \lVert \Sigma^\top  z \rVert$.
    Now, note that $\Sigma^\top z = (s_1 z_1, \ldots, s_r z_r, \ldots , 0, \ldots, 0)$ where $r$ is the rank of $J$ and $s_1, \ldots, s_r$ are the non-zero singular values of $J$.
    Therefore, we have $\lVert \Sigma^\top z \rVert^2 = \sum_{i=1}^r (s_i z_i)^2$, and so overall:
    \[
    \lVert \nabla H_b(f) \rVert^2 \sim \frac{1}{N_t} \sum_{i=1}^r (s_i z_i)^2
    \]
    where each $z_i \sim \cN(0,1)$ independently. Now let $S:= \sum_{i=1}^r s_i^2 = \lVert J \rVert_F^2$, $w_i := s_i^2/S$ and $Y:= \sum_{i=1}^r w_i z_i^2$. Then
    \[
    \lVert \nabla H_b(f) \rVert^2 = \frac{S}{N_t} Y.
    \]
    By Theorem~\ref{thm:frobenius-jacobian-whp}, it suffices to upper bound the probability that $Y$ is too small.
    Indeed, applying Lemma~\ref{lem:technical-gradient} with $s:= 0.0001/e$ gives
    \[
    \Pr[Y \leq 0.0001/e] \leq \frac{0.01}{\sqrt{e}} e^{(1-0.0001/e)/2} \leq 0.01.
    \]
    Overall, we get that with probability at least $0.99$ we have $Y > 0.0001/e$. Therefore, by a union bound with Theorem~\ref{thm:frobenius-jacobian-whp}, we have
    \[
    \lVert \nabla H_b(f) \rVert^2
    =
    \frac{S}{N_t} Y
    \geq
    \frac{\frac{1}{2} (0.1 c^2)^t d_t |G|}{d_t |G|} 0.0001/e
    =
    \frac{0.0001(0.1 c^2)^t }{2e} 
    \]
    with probability at least $0.99 -  2|G| t e^{-0.005 d_{\min}}$, which concludes the proof.
\end{proof}

\section{Gradient robustness} \label{sec:gradient-robust}

We follow the notation of Section~\ref{sec:frobenius-bounds}. For the input $f \in \R^{N_0}$ where $N_\ell = d_\ell |G|$, recall that we denote the output of the entire network as
\[
H_b(f) := \langle u, h^{(t)}(f) \rangle
\]
where $u \sim \cN(0, \frac{1}{N_t} I_{N_t})$.
We also define
\[
\Phi_{t+1}(f) := u^\top, \quad \forall i \in [t]: \Phi_i(f) := \Phi_{i+1}(f) D_i(f) \Lambda_i,
\]
so for $i \in [t]$:
\[
    \Phi_i(f) := u^\top \cdot D_t(f) \Lambda_t \cdot D_{t-1}(f) \cdot \Lambda_{t-1} \cdot \ldots \cdot D_i(f) \cdot \Lambda_i.
\]
Therefore, the gradient $\nabla H_b(f) \in \mathbb{R}^{N_0} \cong L^2(G,\mathbb{R}^d)$ is given by
\[
\nabla H_b(f) = \Phi_1(f).
\]

Denote
\[
M_\infty := \max_{1 \leq i \leq t} \lVert \Phi_i(f) \rVert_\infty, \quad M_s := \max_{1 \leq \ell \leq t} \lVert \Lambda_\ell \rVert.
\]

Denote also $s_\ell (f) := \sigma'(z^{(\ell)}(f)) \in \mathbb{R}^{N_\ell}$. That is, $s_\ell (f)$ is the diagonal of $D_\ell(f)$.
Fix $f \in L^2(G,\mathbb{R}^d)$ and a radius $\Delta > 0$. Let $B(f, \Delta)$ be the $N_0$-ball that is centered at $f$ and has radius $\Delta$. Define

\[
M(f, \Delta) := \max_{1 \leq \ell \leq t} \sup_{f',f'' \in B(f, \Delta)} \lVert s_\ell(f') -s_\ell(f'') \rVert.
\]

We now Prove that $\nabla H_b$ does not change too much in $B(f, \Delta)$ with high probability.
First, we need the following small technical claim.

\begin{lemma} \label{lem:recursive-helper}
    Let $t \in \mathbb{N}$, and let $(x_i)_{i=1}^{t+1}$ satisfy the following recursive inequalities for some constants $a ,b \geq 0$:
    \[
    x_{t+1} = 0, \quad \forall i \in [t]: x_i \leq a x_{i+1} + b.
    \]
    Then:
    \[
    x_1 \leq b (a + 2)^{t+1}.
    \]
\end{lemma}

\begin{proof}
    We first show that for every $i \in [t+1]$ we have
    \[
    x_i \leq b \sum_{k=0}^{t+1-i} a^k.
    \]
    For the base case $i=t+1$ we have:
    \[
    x_{t+1} = 0 \leq b a^0 = b,
    \]
    so it holds.
    For the induction step $i < t+1$, we have
    \begin{align*}
        x_i &\leq a x_{i+1} + b \\
            &\leq a \mleft(b \sum_{k=0}^{t+1-(i+1)} a^k \mright) + b \\
            &=    b \mleft(\sum_{k=1}^{t + 1 -i} a^k +1\mright) \\
            &=    b \sum_{k=0}^{t + 1 -i} a^k
    \end{align*}
    Therefore, for $i=1$ we have
    \[
    x_1
    \leq
    b \sum_{k=0}^{t} a^k \leq b \sum_{k=0}^{t} (a+1)^k \leq b (a+2)^{t+1},
    \]
    as needed.
\end{proof}

We may now prove the desired lemma.
\begin{lemma} \label{lem:gradient-robustness}
    For all $f' \in B(f, \Delta)$ we have
    \[
    \lVert \nabla H_b(f') - \nabla H_b(f) \rVert \leq M_s M_\infty M(f, \Delta) (\lVert \sigma' \rVert_{\infty} M_s + 2)^{t+1}.
    \]
\end{lemma}

\begin{proof}
    For $i \in [t]$ we have
    \begin{align*}
        \lVert \Phi_i(f') - \Phi_i(f) \rVert
        &=
        \lVert \Phi_{i+1}(f') D_i(f') \Lambda_i - \Phi_{i+1}(f) D_i(f') \Lambda_i + \Phi_{i+1}(f) D_i(f') \Lambda_i - \Phi_{i+1}(f) D_i(f) \Lambda_i \rVert \\
        &=
        \lVert (\Phi_{i+1}(f') - \Phi_{i+1}(f)) D_i(f') \Lambda_i + \Phi_{i+1}(f) (D_i(f') - D_i(f)) \Lambda_i \rVert \\
        & \leq
        \lVert \sigma' \rVert_{\infty} M_s \lVert \Phi_{i+1}(f') - \Phi_{i+1}(f) \rVert  + M_s \lVert \Phi_{i+1}(f) (D_i(f') - D_i(f)) \rVert.
    \end{align*}
    Note that $\Phi_{i+1}(f) (D_i(f') - D_i(f)) = \Phi_{i+1}(f) \odot (s_i(f') - s_i(f))$, so
    \[
    \lVert \Phi_{i+1}(f) (D_i(f') - D_i(f)) \rVert
    \leq
    \lVert \Phi_{i+1}(f) \rVert_\infty \lVert s_i(f') - s_i(f) \rVert \leq M_\infty M(f, \Delta).
    \]
    So, we have
    \[
    \lVert \Phi_i(f') - \Phi_i(f) \rVert
    \leq
     \lVert \sigma' \rVert_{\infty} M_s \lVert \Phi_{i+1}(f') - \Phi_{i+1}(f) \rVert + M_s M_\infty M(f, \Delta).
    \]
    It further holds that $\lVert \Phi_{t+1}(f') - \Phi_{t+1}(f) \rVert = 0$. Therefore, we may apply Lemma~\ref{lem:recursive-helper} with
    \[
    x_i := \lVert \Phi_i(f') - \Phi_i(f) \rVert,
    \quad
    a: = \lVert \sigma' \rVert_{\infty} M_s,
    \quad
    b:= M_s M_\infty M(f, \Delta),
    \]
    and get
    \[
    \lVert \nabla H_b(f') - \nabla H_b(f) \rVert
    =
    \lVert \Phi_1(f') - \Phi_1(f) \rVert
    \leq
    M_s M_\infty M(f, \Delta) (\lVert \sigma' \rVert_{\infty} M_s + 2)^{t+1},
    \]
    as required.
\end{proof}

To make the bound of Lemma~\ref{lem:gradient-robustness} useful, we need to establish bounds on the terms appearing in the bound. We prove such bounds in the following lemmas.

\begin{lemma} \label{lem:post-activation-robust}
    Assume w.l.o.g that $\lVert \sigma' \rVert_\infty M_s \geq 1$. Then
    \[
    M(f, \Delta) \leq \lVert \sigma'' \rVert_\infty  \mleft( \lVert \sigma' \rVert_\infty M_s \mright)^{t} (2 \Delta).
    \] 
\end{lemma}

\begin{proof}
    By the mean value theorem, for all $a,b \in \mathbb{R}$ there exists $c\in \mathbb{R}$ such that
    \[
    | \sigma' (a) -\sigma'(b)| \leq |\sigma''(c)|\cdot | a- b|,
    \]
    which implies
    \[
    | \sigma' (a) -\sigma'(b)| \leq \lVert \sigma'' \rVert_\infty | a- b|.
    \]
    Let $f',f'' \in B(f, \Delta)$. Applying the above for every coordinate $j$ of the diagonal of $D_\ell$ gives for every $\ell$:
    \begin{align*}
        \lVert s_\ell(f') - s_\ell(f'') \rVert^2
        & = \sum_j(s_\ell(f')_j - s_\ell(f''))_j ) ^2 \\
        & \leq \lVert \sigma'' \rVert_\infty ^2 \sum | z^{(\ell)}(f')_j - z^{(\ell)}(f''))_j |^2 \\
        & \leq \lVert \sigma'' \rVert_\infty ^2 \lVert z^{(\ell)}(f') -z ^{(\ell)}(f'') \rVert^2.
    \end{align*}
    It remains to bound $\lVert z^{(\ell)}(f') -z ^{(\ell)}(f'') \rVert$. We use the following two inequalities to control how the norm changes in each layer.
    First, note that
    \[
    \lVert z^{(k)}(f') - z^{(k)}(f'') \rVert
    =
    \lVert \Lambda_k (h^{(k-1)}(f') - h^{(k-1)}(f'')) \rVert
    \leq
    \lVert \Lambda_k \rVert \lVert h^{(k-1)}(f') - h^{(k-1)}(f'') \rVert,
    \]
    where the inequality is by Cauchy-Schwarz.
    Second, again from the mean value theorem, we have
    \[
    \lVert h^{(k)}(f') - h^{(k)}(f'') \rVert
    =
    \lVert \sigma (z^{(k)}(f')) - \sigma(z^{(k)}(f'')) \rVert
    \leq
    \lVert \sigma' \rVert_\infty \lVert z^{(k)}(f') - z^{(k)}(f'') \rVert.
    \]
    Combining both inequalities repeatedly from layer $\ell$ to layer $1$ and using the definition of $M_s$ gives
    \[
    \lVert z^{(k)}(f') - z^{(k)}(f'') \rVert
    \leq
    \mleft( \lVert \sigma' \rVert_\infty M_s \mright)^{\ell} \lVert f' - f''\lVert
    \leq
    \mleft( \lVert \sigma' \rVert_\infty M_s \mright)^{\ell} (2 \Delta).
    \]
    Combining this with the bound we derived for $\lVert s_\ell(f') - s_\ell(f'') \rVert$ gives
    \[
    \lVert s_\ell(f') - s_\ell(f'') \rVert
    \leq
    \lVert \sigma'' \rVert_\infty  \mleft( \lVert \sigma' \rVert_\infty M_s \mright)^{\ell} (2 \Delta),
    \]
    which concludes the proof.
\end{proof}

\begin{lemma} \label{lem:small-M-infty}
    We have
    \[
    M_\infty \leq \sqrt{2 \log \mleft( \frac{2t d_{\max} |G|}{0.01} \mright)} \frac{(M_s \lVert \sigma' \rVert_\infty )^t}{\sqrt{N_t}}
    \]
    with probability at least $0.99$.
\end{lemma}

\begin{proof}
    Since $f$ is fixed, denote $\Phi_i := \Phi_i(f)$. Let $B_i$ be the corresponding convolutional part of $\Phi_i$, that is, without $u^\top$. Then by submultiplicativity of spectral norm:
    \[
    \lVert B_i \rVert \leq (M_s \lVert \sigma' \rVert_\infty )^t := K.
    \]
    Now, let $B_{i,j}$ be the $j^{th}$ column of $B_i$. Then $\lVert B_{i,j} \rVert \leq \lVert B_i \rVert \leq K$. Now, since $u \sim \cN(0, \frac{1}{N_t} I_{N_t})$, we have
    \[
    (\Phi_i)_j \sim \cN \mleft(0, \frac{\lVert B_{i,j} \rVert^2}{N_t} \mright)
    \]
    and therefore $\Var((\Phi_i)_j) \leq K^2/N_t$. By the standard gaussian tail bound:
    \[
    \Pr \mleft[ |(\Phi_i)_j| \geq s \frac{K}{\sqrt{N_t}} \mright] \leq 2e^{-s^2/2}.
    \]
    So for $s:= \sqrt{2 \log \mleft( \frac{2t d_{\max} |G|}{\delta} \mright)}$ we have 
    \[
    \Pr \mleft[ |(\Phi_i)_j| \geq s \frac{K}{\sqrt{N_t}} \mright] \leq \frac{\delta}{t d_{\max} |G|}.
    \]
    Now, since $N_\ell \leq d_{\max} |G|$ for all $\ell$, a union bound over all $\ell\in [t],j \in [N_\ell]$ gives
    \[
    \Pr \mleft[ M_\infty > \sqrt{2 \log \mleft( \frac{2t d_{\max} |G|}{\delta} \mright)} \frac{K}{\sqrt{N_t}} \mright] \leq \delta.
    \]
    Choosing $\delta = 0.01$ concludes the proof.
\end{proof}

\section{Deriving the main result} \label{sec:main-result}

In this section we use the tools provided in previous sections, to prove the main result.
Recall that we assume $c + t + \lVert \sigma' \rVert_\infty + \lVert \sigma'' \rVert_\infty = O(1)$.
We first need the following technical claim stating that when the input is not too large, the output is not too large as well.

\begin{lemma} \label{lem:small-output}
    Let $f \in L^2(G,\mathbb{R}^d)$ such that $\lVert f \rVert_\infty \leq 1$. Then with probability at least $0.99$, we have
    \[
    \lvert H_b(f) \rvert \leq 10 \lVert \sigma' \rVert_\infty^{t}.
    \]
\end{lemma}

\begin{proof}
    For each layer $\ell$, we identify a coordinate of $h^{(\ell)}_{g,j}$ by the associated location $g \in G$ and channel $j \in [d_\ell]$. We use the same notation for the pre-activation $z^{(\ell)}_{g,j}$. Denote
    \[
    a_\ell := \max_{g \in G, j \in [d_\ell]} \E \mleft [ \mleft ( h^{(\ell)}_{g,j} \mright)^2 \mright].
    \]
    We first show that the recursive relation $a_\ell \leq \lVert \sigma' \rVert_\infty^2 a_{\ell-1}$ holds. Then, we use it to derive a bound on $\E[H_b(f)^2]$, and finally use Markov's bound to deduce the desired statement.
    
    Fix a layer $\ell$ and $g \in G, j \in [d_\ell]$.  A standard calculation shows that
    \[
    \Var(z^{(\ell)}_{g,j} | h^{(\ell-1)})
    =
    \frac{1}{n_\ell d_{\ell-1}} \sum_{i=1}^{n_\ell} \lVert h^{(\ell-1)}(g_i g) \rVert ^2.
    \]
    Note that since $z^{(\ell)}_{g,j} | h^{(\ell-1)}$ is a centered gaussian (as a linear combination of centered gaussians) we have $\Var(z^{(\ell)}_{g,j} | h^{(\ell-1)}) = \E[(z^{(\ell)}_{g,j})^2 | h^{(\ell-1)}]$. This allows using the tower property to calculate $\E[(z^{(\ell)}_{g,j})^2]$:
    \[
    \E[(z^{(\ell)}_{g,j})^2]
    =
    \E \mleft [ \E \mleft[ (z^{(\ell)}_{g,j})^2 | h^{(\ell-1)}  \mright] \mright]
    =
    \E \mleft [ \Var \mleft[ (z^{(\ell)}_{g,j}) | h^{(\ell-1)}  \mright] \mright]
    =
    \frac{1}{n_\ell d_{\ell-1}} \sum_{i=1}^{n_\ell} \E \mleft[ \lVert h^{(\ell-1)}(g_i g) \rVert^2 \mright].
    \]
    We now upper bound $\E \mleft[ \lVert h^{(\ell-1)}(g_i g) \rVert^2 \mright]$ as
    \[
    \E \mleft[ \lVert h^{(\ell-1)}(g_i g) \rVert^2 \mright]
    =
    \sum_{j=1}^{d_{\ell-1}} \E \mleft[ (h^{(\ell-1)}_{g_i g, j})^2 \mright]
    \leq
    d_{\ell-1} a_{\ell-1}.
    \]
    and conclude that
    \[
    \E[(z^{(\ell)}_{g,j})^2] \leq a_{\ell-1},
    \]
    and thus
     \[
    \E[(h^{(\ell)}_{g,j})^2]
    \leq
    \lVert \sigma' \rVert_\infty^2 \E[(z^{(\ell)}_{g,j})^2]
    \leq
    \lVert \sigma' \rVert_\infty^2 a_{\ell-1}
    \]
    by the assumption $\sigma(0)=0$.
    As the above applies for layer $\ell$ and $(g,j) \in G \times d_{\ell}$, we have
    \[
    a_\ell \leq \lVert \sigma' \rVert_\infty^2 a_{\ell-1}
    \]
    by definition of $a_\ell$, for each layer $\ell$.
    Recall that $h^{(0)} = f$ and $\lVert f \rVert_\infty \leq 1$ and thus $a_0 \leq 1$. Iterating the recursion gives
    \[
    a_t \leq \lVert \sigma' \rVert_\infty^{2t}.
    \]
    We may now use this bound on $a_t$ to bound $\E[(H_b(f))^2]$.
    First, we have:
    \begin{align*}
    \E[(H_b(f))^2 | H(f)]
    &=
    \E[\langle u^\top, H(f) \rangle^2 | H(f)] \\
    &=
    (H(f))^\top \E[u u^\top] H(f) \\
    &=
    (H(f))^\top \mleft ( \frac{1}{N_t} I_{N_t} \mright) H(f)\\
    &=
    \frac{1}{N_t} \lVert H(f) \rVert^2.
    \end{align*}
    Taking expectation on both sides, the tower property gives
    \[
    \E[(H_b(f))^2]
    =
    \frac{1}{N_t} \E \mleft[ \lVert H(f) \rVert^2 \mright]
    =
    \frac{1}{N_t} \sum_{g \in G, j \in [d_t]} \E \mleft[ (h^{(t)}_{g,j})^2 \mright]
    \leq
    a_t
    \leq
    \lVert \sigma' \rVert_\infty^{2t}.
    \]
    Using Markov's bound, we have:
    \[
    \Pr\mleft[(H_b(f))^2 \geq 100 \lVert \sigma' \rVert_\infty^{2t}  \mright]
    \leq
    \frac{\E[(H_b(f))^2]}{100 \lVert \sigma' \rVert_\infty^{2t}}
    \leq
    \frac{\lVert \sigma' \rVert_\infty^{2t}}{100 \lVert \sigma' \rVert_\infty^{2t}}
    =0.01,
    \]
    which concludes the proof.
\end{proof}

We may now prove the main theorem.

\begin{theorem} \label{thm:main}
    Fix $f \in L^2(G,\mathbb{R}^d)$ with $\lVert f \rVert_\infty \leq 1$.
    Then, with probability at least $0.95$, a single step of gradient descent of Euclidean length $O(1)$ starting from $f$ will reach $f' \in L^2(G,\mathbb{R}^d)$ such that $\sign(H_b(f')) \neq \sign(H_b(f))$. 
\end{theorem}

\begin{proof}
    First, we assume that the bad events of Theorem~\ref{thm:sin-val}, Lemma~\ref{lem:gradient}, Lemma~\ref{lem:small-M-infty} and Lemma~\ref{lem:small-output} does not occur. This holds with probability at least $0.96$ with our assumptions from Section~\ref{sec:prem}. Under these assumptions, we have that $M_s, |H_b(f)|, \lVert \nabla H_b(f) \rVert$ are constants, and $M_\infty \to 0$ as $|G| \to \infty$.

    Let $a := 10 \lVert \sigma' \rVert_\infty^{t} \geq |H_b(f)|$, and assume $H_b(f) \geq 0$ without loss of generality. Let $c_1 := \sqrt{\frac{0.0001(0.1 c^2)^t }{2e}} \leq  \lVert \nabla H_b(f) \rVert$. Let $\Delta : = 2a/c_1$, and denote
    \[
        c_2(\Delta) := \sup_{f', f'' \in B(f,\Delta)} \lVert \nabla H_b(f') - \nabla H_b(f'') \rVert.
    \]
    By Lemma~\ref{lem:gradient-robustness} and Lemma~\ref{lem:post-activation-robust} we have
    \[
        c_2(\Delta)
        \leq
        M_\infty M_s M(f,\Delta) (\lVert \sigma' \rVert_\infty M_s +2)^{t+1}
        \leq
        M_\infty M_s \lVert \sigma'' \rVert_\infty  \mleft( \lVert \sigma' \rVert_\infty M_s \mright)^{t+1} (2 \Delta) (\lVert \sigma' \rVert_\infty M_s +2)^{t+1}.
    \]
    Since we assume that $|G|$ is large enough, $M_\infty$ is small enough so that $c_2(\Delta) \leq c_1/3$.

    Now, consider the single gradient descent step
    \[
    f^{(1)} := f - \eta \nabla H_b(f)
    \]
    where $\eta := \frac{2a}{\lVert \nabla H_b(f) \rVert^2}$.
    Now, define the function $\phi: [0, \eta] \to \R$ by $\phi(s) := H_b(f-s \nabla H_b(f))$. By the mean value theorem, there exists $s^\star \in [0,\eta]$ such that
    \[
    \frac{\phi(\eta) - \phi(0)}{\eta - 0} = \phi'(s^\star),
    \]
    which implies
    \begin{equation} \label{eq:gd-step-identity}
    H_b(f^{(1)})
    =
    \phi(\eta)
    =
    \phi(0) + \eta \phi'(s^\star)
    =
    H_b(f) - \eta \langle \nabla H_b(f - s^{\star} \nabla H_b(f)), \nabla H_b(f) \rangle.
    \end{equation}
    So, we want to show that the RHS of the above inequality is negative.
    First, we have
    \begin{align*}
        \langle \nabla H_b(f - s^{\star} \nabla H_b(f)) , \nabla H_b(f) \rangle 
        &=
        \langle \nabla H_b(f) + \nabla H_b(f - s^{\star} \nabla H_b(f)) - \nabla H_b(f) , \nabla H_b(f) \rangle \\
        &=
        \langle \nabla H_b(f),\nabla H_b(f) \rangle + \langle \nabla H_b(f - s^{\star} \nabla H_b(f)) - \nabla H_b(f), \nabla H_b(f) \rangle \\
        & \geq
        \lVert \nabla H_b(f) \rVert^2 - \lVert \nabla H_b(f - s^{\star} \nabla H_b(f)) - \nabla H_b(f)  \rVert \lVert \nabla H_b(f) \rVert \\
        & \geq
        \lVert \nabla H_b(f) \rVert^2 - c_2(\Delta) \lVert \nabla H_b(f) \rVert \\
        & \geq 
        \lVert \nabla H_b(f) \rVert^2  - \frac{c_1}{3} \lVert \nabla H_b(f) \rVert\\
        &=
        \lVert \nabla H_b(f) \rVert^2 \mleft( 1 - \frac{1}{3} \frac{c_1}{\lVert \nabla H_b(f) \rVert} \mright) \\
        & \geq
        \lVert \nabla H_b(f) \rVert^2 \mleft( 1 - \frac{1}{3} \mright) \\
        & =
        \frac{2}{3} \lVert \nabla H_b(f) \rVert^2.
    \end{align*}
    
    To see why the fourth line holds, note that  $f - s \nabla H_b(f) \in B(f,\Delta)$ for all $s \in [0,\eta]$ (and in particular for $s^\star$), since
    \[
    \lVert s \nabla H_b(f) \rVert
    \leq
    \eta \lVert \nabla H_b(f) \rVert
    =
    \frac{2a}{\lVert \nabla H_b(f) \rVert^2} \lVert \nabla H_b(f) \rVert
    =
    \frac{2a}{\lVert \nabla H_b(f) \rVert}
    \leq
    \frac{2a}{c_1}
    =
    \Delta.
    \]
    Therefore, by definition of $c_2(\Delta)$, we have:
    \[
    \lVert \nabla H_b(f - s^{\star} \nabla H_b(f)) - \nabla H_b(f) \rVert \leq c_2(\Delta).
    \]
    The second to last line holds since $\lVert \nabla H_b(f) \rVert \geq c_1$.
    
    Using the inequality 
    \[
    \langle \nabla H_b(f - s^{\star} \nabla H_b(f)) , \nabla H_b(f) \rangle \geq \frac{2}{3} \lVert \nabla H_b(f) \rVert^2
    \]
    we have just proved, combined with the identity \eqref{eq:gd-step-identity} we derived for $H_b(f^{(1)})$ by the mean value theorem, we deduce:
    \begin{align*}
        H_b(f^{(1)})
        &=
        H_b(f) - \eta \langle \nabla H_b(f - s^{\star} \nabla H_b(f)), \nabla H_b(f) \rangle \\
        &\leq
        a - \frac{2a}{\lVert \nabla H_b(f) \rVert^2} \cdot \frac{2}{3}  \lVert \nabla H_b(f) \rVert^2 \\
        &=
        a - \frac{4}{3} a \\
        &=
        -a/3.
    \end{align*}
    Therefore, $\sign(H_b(f^{(1)})) = -1$.
\end{proof}

\section*{Acknowledgments}
The research described in this paper was funded by the European Research Council (ERC) under the European Union’s Horizon 2022 research and innovation program (grant agreement No. 101041711), the Israel Science Foundation (grant number 2258/19), and the Simons Foundation (as part of the Collaboration on the Mathematical and Scientific Foundations of Deep Learning).





    

\bibliographystyle{alphaurl}
\bibliography{bib.bib}

\end{document}